\documentclass[11pt]{article}
\usepackage[ascii]{inputenc} 
\usepackage{amsthm}
\usepackage{setspace}

\newtheorem{theorem}{Theorem}[section]
\newtheorem{corollarytotheproof}[theorem]{Corollary (to the Proof)}
\newtheorem{definition}[theorem]{Definition}
\newtheorem{example}[theorem]{Example}

\newcommand{\bigo}{{\protect\cal O}}

\newcommand{\condition}{\,\mid \:}

\newcommand{\p}{\ensuremath{\mathrm{P}}}
\newcommand{\np}{\ensuremath{\mathrm{NP}}}

\newcommand{\sat}{\ensuremath{\mathrm{SAT}}}
\newcommand{\conp}{\ensuremath{\mathrm{coNP}}}

\newcommand{\pneqnp}{\ensuremath{\p  \neq \np}}

\newcommand{\npintconp}{\ensuremath{\np \cap \conp}}
\newcommand{\npinterconp}{\ensuremath{\npintconp}}

\newcommand{\pisnotinter}{\ensuremath{\p \neq \npinterconp}}
\newcommand{\pisnp}{\ensuremath{\p = \np}}
\newcount\hour  \newcount\minutes  \hour=\time  \divide\hour by 60
\minutes=\hour  \multiply\minutes by -60  \advance\minutes by \time
\def\mmmddyyyy{\ifcase\month\or Jan\or Feb\or Mar\or Apr\or May\or Jun\or Jul\or Aug\or Sep\or Oct\or Nov\or Dec\fi \space\number\day, \number\year}
\def\hhmm{\ifnum\hour<10 0\fi\number\hour :%
  \ifnum\minutes<10 0\fi\number\minutes}

\author{
        Lane A. Hemaspaandra\thanks{Work done in part while 
visiting ETH-Z\"urich and the University of Dusseldorf.}\\
        Department of Computer Science \\
        University of Rochester \\
        Rochester, NY 14627, USA
\and
David E. Narv\'{a}ez\\
College of Computing and Inf.\ Sciences\\
        Rochester Institute of Technology \\
        Rochester, NY 14623, USA}
\date{June 14, 2017;
revised November 1, 2018}

\title{%
Existence versus Exploitation: The Opacity of 
Backdoors and Backbones
Under a Weak Assumption}

\begin{document}
\sloppy

\maketitle

\begin{abstract}

  Backdoors and backbones of Boolean formulas are hidden structural
  properties.  A natural goal, already in part realized, is that
  solver algorithms seek to obtain substantially better performance by
  exploiting these structures.

However, the present paper is not intended to improve 
the performance of SAT solvers, but rather is a cautionary 
paper.
In particular, the theme of this paper 
is that there is a potential chasm 
between the existence of such structures in the Boolean formula and being 
able to effectively exploit them.  This does not mean that 
these structures are not useful to solvers.  It does mean that 
one must be very careful not to assume that it is 
computationally
easy to go from the existence 
of a structure to being able to 
get one's hands on it
and/or being able to exploit the structure.

For example, in this paper we show that, under the
  assumption that $\pneqnp$, there are 
  easily recognizable families of Boolean formulas with strong
  backdoors that are easy to find, yet for which it is 
  hard (in fact, NP-complete) to determine
  whether the formulas are satisfiable.
We also show that, also under the assumption $\pneqnp$, 
there are
easily recognizable sets 
of
  Boolean formulas for which it is 
hard (in fact, NP-complete)
to determine whether they have
a large backbone. 
\end{abstract}

\section{Introduction}
Many algorithms for the Boolean satisfiability problem exploit hidden
structural properties of formulas in order to find a satisfying
assignment or prove that no such assignment exists. These structural
properties are called hidden because they are not explicit in the
input formula. A natural question that arises then is what is the
computational complexity associated with these hidden structures.  In
this paper we focus on two hidden structures: 
backbones and strong
backdoors~\cite{gom-sel-wil:c:backdoors}.

The complexity of decision problems associated with backdoors and
backbones has been studied by 
Nishimura, Ragde, and Szeider~\cite{nis-rag-sze:c:backdoors-horn},
Kilby, Slaney, Thi\'{e}baux, and 
Walsh~\cite{kil-sla-thi-wal:c:backbones-and-backdoors}, and Dilkina,
Gomes, and Sabharwal~\cite{dil-gom-sab:j:backdoors}, among others.

In the present paper, we show that, under the assumption
that $\pneqnp$, 
there are easily recognizable families of formulas with 
strong backdoors
that
are easy to find, yet the problem of determining whether these
formulas are satisfiable
remains hard (in fact, NP-complete).

Hemaspaandra and Narv\'{a}ez~\cite{hem-nar:c:backbones-opacity} showed,
under the (rather strong) assumption that $\pisnotinter$, a
separation between the complexity of finding backbones and that of
finding the values to which the backbone variables must be set.
In the present paper, we also add
to that line of research by showing that, under the
(less demanding) assumption that
$\pneqnp$, there are families of formulas that are easy to recognize
(i.e., they can be recognized by polynomial-time algorithms) yet no
polynomial-time algorithm can, given a formula from the family,
decide whether the formula
has a large backbone (doing so is NP-complete).

Far from being a paper that is intended to speed up SAT solvers, this
is a paper trying to get a better sense of the (potential lack of)
connection between properties existing and being able to get one's
hands on the variables or variable settings that are the ones
expressing the property's existence.  That is, the paper's point is
that there is a potential gap between on one hand the existence of small
backdoors and large backbones, and on the other hand using those to
find satisfying assignments.  Indeed, the paper establishes not just
that (if $\pneqnp$) such gaps exist, but even rigorously proves that
if any NP set exists that is frequently hard (with respect to polynomial-time
heuristics), then sets of our sort
exist that are essentially just as frequently hard; we in effect prove
an inheritance of frequency-of-hardness result, under which our sets
are guaranteed to be essentially as frequently hard as any set in NP is.

Our results admittedly are theoretical results, but they speak both to
the importance of not viewing backdoors or backbones as magically
transparent---we prove that they are in some cases rather opaque---and
to the fact that the behavior we mention likely happens on quite dense
sets; and, further, since we tie this to whether any set is densely
hard, these SAT-solver issues due to this paper 
have now become inextricably linked to the
extremely important, long-open question of how 
resistant to polynomial-time heuristics the hardest sets in NP can
be.\footnote{We mention in passing that there 
are relativized worlds (aka~black-box models) in which NP sets exist for 
which all polynomial-time heuristics 
are asymptotically wrong half the time~\cite{hem-zim:j:balanced};
heuristics basically do no better than one would do by flipping 
a coin to give one's answer.
Indeed, that is known to hold with probability one 
relative to a random oracle, i.e., it holds in all but 
a measure zero set of possible worlds~\cite{hem-zim:j:balanced}.
Although many suspect that the same holds in the real world, 
proving that would 
separate $\np$ from $\p$ in an extraordinarily strong way, and 
currently even proving that $\p$ and $\np$ differ is 
viewed as likely being 
decades (or worse) away~\cite{gas:j:second-p-vs-np-poll}.}
We are claiming that these important hidden
properties---backdoors and backbones---have some rather challenging
behaviors that one must at least be aware of.  Indeed, what is most
interesting about this paper is likely not the theoretical
constructions themselves, but rather the behaviors that those
constructions prove must exist unless $\pisnp$.  We feel that
knowing that those behaviors cannot be avoided unless $\pisnp$ is of
potential interest to both AI and theory.  Additionally, the 
behavior in one of our results is closely connected to the 
deterministic time complexity of SAT; in our result
(Theorem~\ref{t:backdoors})
about easy-to-find hard-to-assign-values-to backdoors, we show
that the backdoor size bound in our theorem cannot be improved 
even slightly unless NP is contained in subexponential time.

The rest of this paper is organized as follows. Section
\ref{definitions} defines the notation we will use throughout this
paper. Sections~\ref{backdoors} 
and~\ref{backbones} 
contain our
results related to backdoors and backbones, respectively. Finally,
Section \ref{conclusion} adds some concluding remarks.

\section{Definitions and Notations}
\label{definitions}
For a Boolean formula $F$, we denote by $V(F)$ the set of
variables appearing in $F$. 

Adopting the notations of 
Williams, Gomes, and Selman~\cite{gom-sel-wil:c:backdoors},
we use the following.
A partial assignment of $F$ is a function
$a_S: S\rightarrow \{\textsc{True},\textsc{False}\}$ that
assigns Boolean values to the variables in a set
$S\subseteq V(F)$. For a Boolean value
$v\in\{\textsc{True},\textsc{False}\}$ and a variable
$x\in V(F)$, the notation $F[x/v]$ denotes the
formula $F$ after replacing every occurrence of $x$ by $v$ and
simplifying. This extends to partial assignments,
e.g., to $F[a_S]$, in the natural way.

For a finite set $A$, $\|A\|$ denotes $A$'s cardinality. 
For any string $x$, $|x|$ denotes the length of (number of characters of)
$x$.

For each set $T$ and each natural number $n$, $T^{\leq n}$ 
denotes the set of all strings in $T$ whose length is 
less than or equal to $n$.
In particular, 
$(\Sigma^*)^{\leq n}$ denotes the strings of length at most $n$,
over the alphabet $\Sigma$.

\section{Results on Backdoors to CNF Formulas}
\label{backdoors}
In this section we focus on Boolean formulas in conjunctive normal
form, or CNF\@.  A CNF formula is a conjunction of disjunctions, and the
disjunctions are called the \emph{clauses} of the formula.
Following Dilkina, Gomes, and
Sabharwal~\cite{dil-gom-sab:j:backdoors}, we define satisfiability of
CNF formulas using the language of set theory. 
This is done by
formalizing the intuition that, in order for an assignment to satisfy
a CNF formula, it must set at least one literal in every clause to
\textsc{True}. One can then define a CNF formula $F$ to be a collection
of clauses, each clause being a set of literals.  $F\in\sat$ if and
only if there exists an assignment $a_{V(F)}$ such that for
all clauses $C\in F$ there exists a literal $l\in C$ such that
$a_{V(F)}$ assigns $l$ to \textsc{True}. 
Under this formalization, to be in harmony with the standard
conventions that 
the truth value of the empty conjunctive (resp., disjunctive)
formula is \textsc{True} (resp., \textsc{False}),
$F$ must be taken to be in
$\sat$ if $F$ is empty, and $F$ must be taken to 
be in $\overline{\sat}$ if
$\emptyset\in F$ (since the empty CNF formula must be taken 
to be \textsc{False} as a consequence of the fact that 
the empty disjunctive formula is taken to be \textsc{False}); 
these two cases are called, respectively, 
$F$ being trivially \textsc{True} and $F$ being trivially 
\textsc{False} (as the conventions 
as just mentioned 
put these cases not just in $\sat$ and $\overline{\sat}$ but 
fix the truth values of the represented formulas to 
be \textsc{True} and \textsc{False}).
We can also formalize simplification using this
notation: after assigning a variable $x$ to \textsc{True} (resp.,
\textsc{False}), the formula is simplified by removing all clauses
that contain the literal $x$ (resp., $\overline{x}$) and removing the
literal $\overline{x}$ (resp., $x$) from the remaining clauses. This
formalization extends to simplification of a formula over a partial
assignment in the natural way.

\begin{example}
\label{e:cnf}
 Consider the CNF formula
 $F=(x_1\lor\overline{x_2}\lor\overline{x_3}\lor x_5)\land(x_1\lor
 x_2\lor x_4\lor
 x_5)\land(x_3\lor\overline{x_4})\land(\overline{x_1}\lor x_2\lor
 x_3\lor x_5)$. We can express this formula in our set theory
 notation as $F=\{\{x_1,\overline{x_2},\overline{x_3}, x_5\},\{x_1,
 x_2, x_4, x_5\}, \{x_3, \overline{x_4}\}, \{\overline{x_1}, x_2, x_3,
 x_5\}\}$. Suppose we assign $x_3$ to \textsc{False} and $x_4$ to
 \textsc{True}, we have
 $F[x_3/\textsc{False},x_4/\textsc{True}]=\{\emptyset,
 \{\overline{x_1}, x_2, x_5\}\}$, which is unsatisfiable because it
 contains the empty set.
\end{example}

Since CNF-SAT (the satisfiability problem restricted to CNF formulas)
is well-known to be $\np$-complete, a polynomial-time algorithm to
determine the satisfiability of CNF formulas is unlikely to exist.
Nevertheless, there are several restrictions of CNF formulas for which
satisfiability can be decided in polynomial time.
When a formula does not belong to any of these restrictions, it may have
a set of variables that, once the formula is simplified over a partial
assignment of these variables, the resulting formula belongs to one of
these tractable restrictions. A formalization of this idea is the
concept of backdoors.

\begin{definition}[Subsolver~\cite{gom-sel-wil:c:backdoors}]
  A polynomial-time 
algorithm $A$ is a \emph{subsolver} if, for each input formula $F$, $A$
  satisfies the following conditions.
  \begin{enumerate}
  \item $A$ either rejects the input $F$ (this indicates
that it declines to 
make a statement as to whether $F$ is satisfiable) 
or \emph{determines} $F$
(i.e., $A$ returns a satisfying assignment if $F$ is satisfiable and 
$A$ proclaims $F$'s
unsatisfiability if $F$ is unsatisfiable).
  \item If $F$ is trivially \textsc{True} 
$A$ determines $F$, and 
if $F$ is trivially \textsc{False} 
$A$ determines $F$.
  \item If $A$ determines $F$, then for each variable $x$ and each value
    $v$, $A$ determines $F[x/v]$.
  \end{enumerate}
\end{definition}

\begin{definition}[Strong Backdoor~\cite{gom-sel-wil:c:backdoors}]
  For a Boolean formula $F$, a nonempty subset $S$ of its variables is
  a \emph{strong backdoor} for a subsolver $A$ if, for all partial
  assignments $a_S$, $A$ determines $F[a_S]$
(i.e., if $F[a_S]$ is 
satisfiable $A$ 
returns a satisfying assignment and if $F[a_S]$ is 
unsatisfiable $A$ proclaims its unsatisfiability).
\end{definition}

Many examples of subsolvers can be found in the literature (for
instance, in Table~1 of~\cite{dil-gom-sab:j:backdoors}).  The
subsolver that is of particular relevance to this paper is the
\emph{unit propagation subsolver}, which focuses on \emph{unit
  clauses}. Unit clauses are clauses with just one literal. They play
an important role in the process of finding models (i.e., satisfying
assignments) because the literal
in that clause must be set to \textsc{True} in order to find a
satisfying assignment. The process of finding a model by searching for
a unit clause (for specificity and to ensure 
that it runs in polynomial time, let us say that our unit propagation
subsolver always focuses on 
the unit clause in the current formula whose encoding is the 
lexicographically least 
among the encodings of all unit clauses in the current formula), 
fixing the value of the variable in the unit clause,
and simplifying the formula resulting from that assignment is known in
the 
satisfiability 
literature as unit propagation. Unit propagation is an
important building block in the seminal DPLL algorithm for
$\sat$~\cite{dav-put:j:computing-proc-quant-theory,dav-log-lov:j:machine-program-theorem-proving}. Notice that the CNF
formulas whose satisfiability can be decided by just applying unit
propagation iteratively constitute a tractable restriction of $\sat$.
The unit propagation subsolver attempts to decide the satisfiability
of an input formula by using only unit propagation and empty clause
detection. If satisfiability cannot be decided this way, the subsolver rejects
the input formula.
Szeider~\cite{sze:j:backdoors-dll} has classified the parameterized
complexity of finding backdoors with respect to the unit propagation
subsolver.

\begin{example}
 Consider the formula $F$ from Example~\ref{e:cnf}. We will show that
 $\{x_1,x_3,x_5\}$ is a strong backdoor of $F$ with respect to the
 unit propagation subsolver by analyzing the possible assignments of
 these variables. Suppose $x_1$ is assigned to \textsc{True} and
 notice
 $F[x_1/\textsc{True}]=\{\{x_3,\overline{x_4}\},\{x_2,x_3,x_5\}\}$. From
 there it is easy to see that if $x_3$ is set to \textsc{True}, the
 resulting formula after simplification is trivially satisfiable. If
 $x_3$ is set to \textsc{False}, assigning $x_5$ to \textsc{True}
 yields the formula $\{\{\overline{x_4}\}\}$ after simplification and
 the satisfiability of this formula can be determined by the unit
 propagation subsolver. Assigning $x_5$ to \textsc{False} yields a
 formula with two unit clauses, $\{\{\overline{x_4}\},\{x_2\}\}$. The
 unit propagation subsolver will pick the unit clause
 $\{x_2\}$,\footnote{Here we assume that a clause $\{x\}$ precedes a
   clause $\{y\}$ in lexicographical order if $x$ precedes $y$ in
   lexicographical order.} assign the truth value of $x_2$ and
 simplify, and will then pick the (sole) remaining unit clause,
 $\{\overline{x_4}\}$, and assign the truth value of $x_4$ and simplify to
 obtain a trivially satisfiable formula. Now suppose $x_1$ is assigned
 to \textsc{False} and notice
 $F[x_1/\textsc{False}]=\{\{\overline{x_2},\overline{x_3},x_5\},\{x_2,x_4,x_5\},\{x_3,\overline{x_4}\}\}$. If
 we now assign $x_3$ to \textsc{True}, notice
 $F[x_1/\textsc{False},x_3/\textsc{True}]=\{\{\overline{x_2},x_5\},\{x_2,x_4,x_5\}\}$. If
 we assign $x_5$ to \textsc{True} $F$ simplifies to a trivially
 satisfiable formula. If we assign $x_5$ to \textsc{False}, the
 formula simplifies to $\{\{\overline{x_2}\},\{x_2,x_4\}\}$. The unit
 propagation subsolver will pick the unit clause $\{\overline{x_2}\}$,
 assign the truth value of $x_2$, and the resulting formula after
 simplification will be $\{\{x_4\}\}$ whose satisfiability can be
 determined by the unit propagation subsolver. If we assign $x_3$ to
 \textsc{False}, notice
 $F[x_1/\textsc{False},x_3/\textsc{False}]=\{\{x_2,x_4,x_5\},\{\overline{x_4}\}\}$. If
 we now assign $x_5$ to \textsc{True} and simplify, the resulting
 formula would be $\{\{\overline{x_4}\}\}$ whose satisfiability can be
 determined by the unit propagation subsolver. If we assign $x_5$ to
 \textsc{False} and simplify, the resulting formula would contain the
 unit clause $\{\overline{x_4}\}$. The unit propagation subsolver
 would then set the value of $x_4$ to \textsc{False} and simplify,
 yielding the formula $\{\{x_2\}\}$, whose satisfiability can also be
 determined by the unit propagation subsolver.
 
 It should be clear from the case analysis above that just setting the
 values of $x_1$ and $x_3$ is not enough for the unit propagation
 subsolver to always be able to determine the satisfiability of the
 resulting formula. In fact, a similar analysis done on every 2-element 
subset and
 every 3-element subset of $V(F)$---which we do not write out 
 here---shows
 that $\{x_1,x_3,x_5\}$ is actually the
 smallest strong backdoor of $F$ with respect to the unit propagation
 subsolver.
\end{example}

We are now ready to prove our main result about backdoors:
Under
the assumption that $\pneqnp$, there are families of Boolean formulas
that are easy to recognize and have strong unit propagation backdoors
that are easy to find, yet deciding whether the formulas in these
families are satisfiable 
remains $\np$-complete.

\begin{theorem} %
  \label{t:backdoors}
  If $\pneqnp$, for each $k\in\{1, 2, 3,\ldots\}$ there is
  a set $A$ of Boolean formulas such that all the following hold.
  \begin{enumerate}
  \item $A\in\p$ and $A\cap\sat$ is $\np$-complete.
  \item\label{i:backdoor-prop} Each formula $G$ in $A$ has a strong
    backdoor $S$ with respect to the unit propagation subsolver, with
    $\|S\|\leq\|V(G)\|^\frac{1}{k}$.
  \item There is a polynomial-time algorithm that, given
    $G\in A$, finds a strong backdoor having the 
    property stated in item~\ref{i:backdoor-prop} of this theorem.
  \end{enumerate}
\end{theorem}

\begin{proof}
For $k=1$ the theorem is trivial, so we henceforward consider 
just the case where $k\in\{2, 3,\ldots\}$.
  Consider
(since in the following set definition 
$F$ is specified as being 
in CNF, we can safely start 
the following with ``$F \wedge$'' rather than
for example ``$(F) \wedge$'')
$A \in {\rm P}$ defined by
$$A =\{F\wedge(\mathtt{new}_1\wedge\cdots\wedge\mathtt{new}_{
      \|V(F)\|^k-\|V(F)\|}
  )\condition\textnormal{$F$ is a CNF formula}\},$$ where
$\mathtt{new}_i$ is the 
$i$th 
(in lexicographical order) legal
variable name that does not appear in $F$. For instance, if $F$
contains literals $\overline{x_1}$, $x_2$, $x_3$, and
$\overline{x_3}$, and if our legal variable universe is
$x_1, x_2, x_3, x_4,\ldots$, then $\mathtt{new}_1$ would be
$x_4$.
The backdoor is the set
of variables of $F$, which can be found in polynomial time by
parsing. It is clear that the formula resulting from simplification
after assigning values to all the variables of $F$ only has unit
clauses and potentially an empty clause, so satisfiability for this
formula can be decided by the unit propagation subsolver. Finally, it
is easy to see that
$F\wedge(\mathtt{new}_1\wedge\cdots\wedge\mathtt{new}_{
    \|V(F)\|^k-\|V(F)\|}
)\in\sat\Leftrightarrow F\in\sat$ so,
since the formula-part that is being postpended to $F$
can easily be polynomial-time constructed given $F$,
under the assumption that
$\pneqnp$ deciding satisfiability for the formulas in $A$ is hard.
\end{proof}
We mention in passing that one can change ``Boolean'' to
``Boolean CNF'' in Theorem~\ref{t:backdoors}'s statement, via adjusting
appropriately the use of parentheses in the proof's definition
of $A$ to ensure that $A$ itself is in CNF whenever $F$ is.

Let us address two natural worries the reader might have regarding
Theorem~\ref{t:backdoors}.  First, the reader might worry that the
hardness spoken of in the theorem occurs very infrequently (e.g.,
perhaps except for just 
one string at every double-exponentially spaced length everything 
is easy).  That is, are 
we giving a worst-case result that deceptively hides a low 
typical-case complexity?   We are not (unless all of NP
has easy typical-case complexity):
we show that if 
any 
set in NP is frequently hard 
with respect to polynomial-time
heuristics, then a set of our sort is almost as frequently hard
with respect to polynomial-time
heuristics.  We will show this as Theorem~\ref{t:freq-backdoors}.

But first let us address a different worry.  Perhaps some readers will
feel that the fact that Theorem~\ref{t:backdoors} speaks of backdoors
of size bounded by a fixed $k$th root in size is a weakness, and that
it is disappointing that the theorem does not establish its same
result for a stronger bound such as ``constant-sized backdoors,'' or
if not that then polylogarithmic-sized, or if
not that then at least ensuring that not just each fixed root is
handled in a separate construction/set
but that a single construction/set should 
assert/achieve the case of a growth rate that 
is asymptotically less than every
root.  Those are all fair and natural to wonder about.  However, we claim
that not one of those improvements of 
Theorem~\ref{t:backdoors}
can be proven without revolutionizing the deterministic speed of 
SAT\@.   In particular, the following result holds, showing that 
those  three cases would respectively put NP into P, 
quasipolynomial time, and subexponential time.

\begin{theorem}
\begin{enumerate}

\item{[Constant case]}
Suppose there is a 
$k \in \{1,2,3,\ldots\}$
and 
  a set $A$ of Boolean formulas such that all the following hold:
(a)~$A\in\p$ and $A\cap\sat$ is $\np$-complete;
(b)~each formula $G$ in $A$ has a strong
    backdoor $S$ with respect to the unit propagation subsolver, with
    $\|S\|\leq k$; and 
(c)~there is a polynomial-time algorithm that, given
    $G\in A$, finds a strong backdoor having the 
    property stated in item~(b).  %
Then $$\pisnp.$$

\item{[Polylogarithmic case]}
Suppose there is a function $s(n)$, with $s(n) = (\log n)^{\bigo(1)}$, and 
  a set $A$ of Boolean formulas such that all the following hold:
(a)~$A\in\p$ and $A\cap\sat$ is $\np$-complete;
(b)~each formula $G$ in $A$ has a strong
    backdoor $S$ with respect to the unit propagation subsolver, with
    $\|S\|\leq s({\|V(G)\|})$; and 
(c)~there is a polynomial-time algorithm that, given
    $G\in A$, finds a strong backdoor having the 
    property stated in item~(b). %
Then $\np$ is in quasipolynomial time, i.e., $$\np  \subseteq 
        \bigcup_{c > 0} {\rm DTIME}[2^{(\log n)^c}].$$

\item{[Subpolynomial case]}
Suppose there is a polynomial-time computable function $r$ and 
  a set $A$ of Boolean formulas such that all the following hold:
(a)~for each $k \in \{1,2,3,\ldots\}$, $r(0^n) = \bigo(n^\frac{1}{k})$;
(b)~$A\in\p$ and $A\cap\sat$ is $\np$-complete;
(c)~each formula $G$ in $A$ has a strong
    backdoor $S$ with respect to the unit propagation subsolver, with
    $\|S\|\leq r(0^{\|V(G)\|})$; and 
(d)~there is a polynomial-time algorithm that, given
    $G\in A$, finds a strong backdoor having the 
    property stated in item~(c).  %
Then $\np$ is in subexponential time, i.e., $$\np  \subseteq 
        \bigcap_{\epsilon > 0} {\rm DTIME}[2^{n^\epsilon}].$$
      \end{enumerate}
    \end{theorem}
      
We can see this as follows.  
Consider the ``Constant case''---the first part---of the above theorem.
Let $k$ be the constant of that part.
Then there are at most
${N
\choose k}$ ways of choosing $k$ of the variables of a given Boolean
formula of $N$ bits (and thus of at most $N$ variables).
And for each of those ways, we can try all $2^{k}$ possible ways of
setting those variables.  This is $\bigo(N^{k})$ items to test---a
polynomial number of items.  If the formula is satisfiable, then via
unit propagation one of these must yield a satisfying assignment (in 
polynomial time). Yet
the set $A \cap \sat$ was NP-complete by the first condition of the
theorem.  So we have that $\pisnp$, since we just gave a
polynomial-time algorithm for $A \cap \sat$.
The other three cases are analogous (except in 
the final case, we in the theorem needed to 
put in the indicated polynomial-time 
constraint on the bounding function $r$ since otherwise it could be 
badly behaved; that issue doesn't affect the second part of the theorem
since even a badly behaved function $s$ of the second part is bounded above
by a simple-to-compute function $s'$ satisfying
$s'(n) = (\log n)^{\bigo(1)}$ and we can use $s'$ in place 
of $s$ in the proof).

Even the final part of the above theorem, which is the part 
that has the weakest hypothesis, implies that NP is in 
subexponential time.
However, it is widely suspected
        that the NP-complete sets lack subexponential-time algorithms.
And so we have 
        established that 
	the $n^{1/k}$ growth, which we \emph{do} 
prove 
in Theorem~\ref{t:backdoors},
	is the smallest bound in part~\ref{i:backdoor-prop} of that result 
        that one can hope to prove
Theorem~\ref{t:backdoors}
for without having to as a side effect put NP into a
	deterministic time class so small that we would have a
	revolutionarily fast deterministic algorithm for $\sat$.

Moving on, we now, as promised above, 
address the frequency of hardness of the sets we define
in Theorem~\ref{t:backdoors}, and show that if any set in NP is 
frequently hard then a set of our type is almost-as-frequently hard.
(Recall that, when $n$'s universe is 
the naturals as it is in the following theorem, ``for almost every $n$'' means 
``for all but at most a finite number of natural numbers $n$.'')
We will say that a (decision) 
algorithm errs with respect to $B$ on an input $x$
if the algorithm disagrees with $B$ on $x$, i.e., 
if 
the algorithm accepts $x$ yet $x \not\in B$ or 
the algorithm rejects $x$ yet $x \in B$.

\begin{theorem} %
  \label{t:freq-backdoors}
  If $h$ is any nondecreasing function and for some set $B\in\np$ it
  holds that each polynomial-time algorithm 
  errs with respect to $B$, 
  at infinitely many lengths $n$ (resp., for almost every length $n$),
  on at least $h(n)$ of the inputs up to that length, then
  there will exist an $\epsilon>0$ and a set $A\in\p$ of Boolean
  formulas satisfying the conditions 
  of Theorem~\ref{t:backdoors}, yet
  being such that each polynomial-time algorithm $g$, at infinitely
  many lengths $n$ (resp., for almost every length $n$), will fail to determine
  membership in $A\cap\sat$ for at least $h(n^\epsilon)$
  inputs of length at most $n$.
\end{theorem}

Before getting to the proof of this theorem, let us give 
concrete examples that give a sense about what the theorem is saying
about density transference.  
It follows from Theorem~\ref{t:freq-backdoors} that
if there exists 
even one NP set such that each polynomial-time heuristic 
algorithm asymptotically errs
exponentially often up to each length (i.e., has 
$2^{n^{\Omega(1)}}$ errors), then
there are sets of our form that in the same 
sense fool each polynomial-time heuristic 
algorithm exponentially often.  
As a second example, 
it follows from Theorem~\ref{t:freq-backdoors} that
if there exists 
even one NP set such that each polynomial-time heuristic 
algorithm asymptotically errs
quasipolynomially often up to each length (i.e., has 
$n^{(\log n)^{\Omega(1)}}$ 
errors), then
there are sets of our form that in the same sense 
fool each polynomial-time heuristic 
algorithm quasipolynomially often.
Since almost everyone suspects that
some NP sets are quasipolynomially and indeed even 
exponentially densely hard, one must with 
equal strength of belief suspect that there are sets of our form
that are exponentially densely hard.

\begin{proof}[Proof of Theorem~\ref{t:freq-backdoors}]
For conciseness and to avoid repetition, we build  this proof 
on 
top of a proof 
(namely, of 
Theorem~\ref{t:freq-backbones})
that we will give later in the paper.  That later proof does 
not rely directly or indirectly 
on the present theorem/proof, so there is no circularity at issue
here.  However, 
readers wishing to read the present proof should 
probably delay doing that until after they have first read that later proof.

  We define $r_B$ as in the proof of Theorem~\ref{t:freq-backbones}
  (the $r_B$ given there draws on a construction from
  Appendix~A of~\cite{hem-nar:t4:backbones-opacity}, and due to that 
  construction's properties outputs only conjunctive normal form 
formulas). For a given $k$, we define
$$A=\{r_B(x)\wedge
(\mathtt{new}_1\wedge\cdots\wedge\mathtt{
      new}_{
      \|V(r_B(x))
\|^k-\|V(r_B(x)
        )\|})\condition x\in\Sigma^*\},$$ and
since
$r_B(x)\wedge(\mathtt{new}_1\wedge\cdots\wedge\mathtt{new}_{
    \|V(r_B(x))\|^k-
\|V(r_B(x)
      )\|})\in\sat\Leftrightarrow
r_B(x)\in\sat$ 
and 
$r_B(x)\in\sat
\Leftrightarrow x\in B$, we can now proceed as
in the proof of Theorem~\ref{t:freq-backbones}, since here too the tail's
length is 
polynomially bounded.
\end{proof}

\section{Results on Backbones}
\label{backbones}
For the sake of completeness, we start this section by restating the
definition of backbones as presented by Williams, Gomes, and
Selman~\cite{gom-sel-wil:c:backdoors}. We restrict ourselves to the
Boolean domain, since we only deal with Boolean formulas in this
paper.

\begin{definition}[Backbone~\cite{gom-sel-wil:c:backdoors}]
  For a Boolean formula $F$, a subset $S$ of its variables is a
  \emph{backbone} if there is a unique partial assignment $a_S$ such
  that $F[a_S]$ is satisfiable.
  \label{def:backbone}
\end{definition}

The \emph{size} of a backbone $S$ is the number of variables in
$S$. One can readily see from Definition~\ref{def:backbone} that all
satisfiable formulas have at least one backbone, namely, the empty
set. This backbone is called the \emph{trivial} backbone, while
backbones of size at least one are called \emph{nontrivial}
backbones. 
It follows from
Definition~\ref{def:backbone} that unsatisfiable formulas do not
have backbones. 
Note also that some satisfiable formulas have no 
nontrivial backbones, e.g., $x_1 \lor x_2 \lor x_3$ is 
satisfiable but has no nontrivial backbone.

\begin{example}
 Consider the formula $F=x_1\land(x_1\leftrightarrow
 \overline{x_2})\land(x_2\leftrightarrow x_3)\land(x_2\lor x_4\lor
 x_5)$. Any satisfying assignment of $F$ must have $x_1$ set to
 \textsc{True}, which in turn constrains $x_2$ and $x_3$. Then
 $\{x_1,x_2,x_3\}$ is a backbone of $F$, as is any subset of this
 backbone. It is also easy to see that $\{x_1,x_2,x_3\}$ is the
 largest backbone of this formula since the truth values of $x_4$ and
 $x_5$ are not entirely constrained in $F$ (since $F$ in effect 
is---once one applies the just-mentioned forced assignments---$x_4 \lor x_5$).
\end{example}

Our first result states that 
if $\pneqnp$ then there are families of
Boolean formulas that are easy to recognize, with the property that
deciding whether a formula in these families has a large
backbone is NP-complete (and so is hard).
As a corollary to its proof, we have that 
if $\pneqnp$ then there are families of
Boolean formulas that are easy to recognize, with the property that
deciding whether a formula in these families has a nontrivial
backbone is 
NP-complete (and so is hard).\footnote{We have not been able to find 
Corollary (to the Proof)~\ref{c:backbones-1} in the literature.
Certainly, two things that on their surface might seem
to be the claim we are making in 
Corollary (to the Proof)~\ref{c:backbones-1} are either trivially true
or are in the literature.  However, upon closer inspection they turn out to 
be quite different from our claim.

In particular, if one removes the word ``nontrivial''
from 
Corollary (to the Proof)~\ref{c:backbones-1}'s statement, and one 
is in the model in which every satisfiable formula is considered to 
have the empty collection of variables as a backbone and every 
unsatisfiable formula is considered to have no backbones, then 
the thus-altered version of 
Corollary (to the Proof)~\ref{c:backbones-1}
is clearly true, 
since if one with those changes takes $A$ to be the set of all 
Boolean formulas, then the theorem degenerates to the statement 
that if $\pneqnp$, then SAT is (NP-complete, and) not in $\p$.

Also, it is stated in 
Kilby 
et al.~\cite{kil-sla-thi-wal:c:backbones-and-backdoors}
that finding a backbone of CNF 
formulas is NP-hard.  However, though this might seem to be our result,
their claim and model differ from ours 
in many ways, making this a quite
different issue.  
First, their hardness refers to Turing reductions
(and in contrast our paper is 
about many-one reductions and many-one completeness). 
Second, 
they are not even speaking of NP-Turing-hardness---much
less NP-Turing-completeness---in the standard sense since
their model is assuming
a function reply from the oracle rather than 
having a set
as 
the oracle.
Third, even their
notion of backbones is quite different as it (unlike the influential
Williams, Gomes, and Selman 2003 paper~\cite{gom-sel-wil:c:backdoors}
and our paper) in effect requires 
that the function-oracle gives back both a variable \emph{and its
setting}.  Fourth, our claim is about \emph{nontrivial} backbones.}

\begin{theorem} %
  \label{t:backbones-2}
  For any real number $0<\beta < 1$, there is a set
  $A\in\p$ of Boolean formulas such that the language
$$L_A=\{F\condition\textnormal{$F\in A$ and $F$ has a 
backbone 
    $S$ with
    $\|S\|\geq\beta\|V(F)\|$}\}$$
  is $\np$-complete (and so
    if $\pneqnp$ then $L_A$ is not in $\p$).
\end{theorem}

\begin{corollarytotheproof} %
  \label{c:backbones-1}
  There is a set
  $A\in\p$ of Boolean formulas such that the language
$$L_A=\{F\condition\textnormal{$F\in A$ and $F$ has a 
nontrivial 
backbone 
    $S$}\}$$
  is $\np$-complete
      (and so
    if $\pneqnp$ then $L_A$ is not in $\p$).

\end{corollarytotheproof}

\begin{proof}[Proof of Theorem~\ref{t:backbones-2} and
Corollary~\ref{c:backbones-1}]
We will first prove Theorem~\ref{t:backbones-2}, and then will note that 
Corollary~\ref{c:backbones-1}
follows easily as a corollary to the proof/construction.

So fix a $\beta$ from Theorem~\ref{t:backbones-2}'s statement.
For each Boolean formula $G$, let
  $$q(G)=\left\lceil\frac{\beta\|V(G)\|}{1-\beta}
  \right\rceil.$$ 
Define
$$A=\{
(G)\wedge(\mathtt{new}_1\wedge\mathtt{new}_2\wedge
\cdots\wedge\mathtt{new}_{q(G)})\condition\textnormal{
    $G$ is a Boolean formula having at least one variable}\},$$
where, as in the proof of Theorem~\ref{t:backdoors}, we define
$\mathtt{new}_i$ as the $i$th variable that does not appear in $G$.
Note that
$\mathtt{new}_1\wedge\mathtt{new}_2\wedge\cdots\wedge\mathtt{new}_{
  q(G)}$ is a backbone if and only if $G\in\sat$, thus
under the assumption that $\pneqnp$ and keeping in mind 
that for zero-variable formulas satisfiability is easy to decide,
it follows that no polynomial-time algorithm can
decide $L_A$, since the size of this backbone is $q(G)>0$,
which by our definition of $q$ 
will satisfy the condition $\|S\| \geq \beta \|V(F)\|$.  
Why does it satisfy that condition?  
$\|S\|$ here is $q(G)$.  And 
$\|V(F)\|$ here, since 
$F$ is the formula 
$(G)\wedge(\mathtt{new}_1\wedge\mathtt{new}_2\wedge
\cdots\wedge\mathtt{new}_{q(G)})$, equals $\|V(G)\| + q(G)$.
So the condition is claiming that 
$q(G) \geq \beta (\|V(G)\| + q(G))$, 
i.e., that $q(G) \geq 
{\beta \over (1-\beta)}
\|V(G)\|$,
which indeed holds in light of 
the definition of $q$.
And why do we claim that no polynomial-time algorithm
can decide $L_A$?  Well, note that 
$\sat$ many-one 
polynomial-time reduces to $L_A$ via the reduction
$g(H)$ that equals 
some fixed string in $L_A$ if $H$ is in $\sat$ and $H$ has zero variables 
and that equals 
some fixed string in $\overline{L_A}$ 
if  $H$ is not in $\sat$ and $H$ has zero variables (these two
cases are 
included merely to handle degenerate things such as $\textsc{True} \lor \textsc{False}$
that can occur if we allow \textsc{True} and \textsc{False} as atoms in our propositional 
formulas), and 
that equals 
$$(H)\wedge(\mathtt{new}_1\wedge\mathtt{new}_2\wedge
\cdots\wedge\mathtt{new}_{q(H)})$$
otherwise (the above formula is $H$ conjoined 
with a large number of new variables). 
Since $L_A$ is in $\np$,\footnote{If one just 
looks at the definition of $L_A$, one might worry that 
$L_A$ might have only $\np^\np$ as an obvious upper bound.
However, as noted above our particular choice of $A$ 
ensures that 
$\mathtt{new}_1\wedge\mathtt{new}_2\wedge\cdots\wedge\mathtt{new}_{
  q(H)}$ is a backbone 
of 
$(H)\wedge(\mathtt{new}_1\wedge\mathtt{new}_2\wedge
\cdots\wedge\mathtt{new}_{q(H)})$
if and only if $H\in\sat$; and that 
makes clear that our set is indeed in $\np$.
}
we have that it is 
$\np$-complete, and since $\p\neq\np$ was part of the theorem's
hypothesis, $L_A$ cannot be in $\p$.

The above proof establishes 
Theorem~\ref{t:backbones-2}.
Corollary~\ref{c:backbones-1} 
follows immediately from the proof/construction
of 
Theorem~\ref{t:backbones-2}.  Why?  
The set $A$ from the proof of 
Theorem~\ref{t:backbones-2}
is constructed
in such a way that each of its potential members 
$(G)\wedge(\mathtt{new}_1\wedge\mathtt{new}_2\wedge
\cdots\wedge\mathtt{new}_{q(G)})$ (where $G$ is a Boolean 
formula having at least one variable)
either has no 
nontrivial backbone (indeed, no backbone)
or has a backbone of size at least 
$\beta(\|V(G)\|)$.
Thus the issue of backbones 
that are nontrivial but smaller than 
$\beta(\|V(F)\|)$, where 
$F$ is 
$(G)\wedge(\mathtt{new}_1\wedge\mathtt{new}_2\wedge
\cdots\wedge\mathtt{new}_{q(G)})$,
does not
cause a problem under the construction.
That is, 
our $A$ (which itself is dependent on the value 
of $\beta$ one is interested in) is such that
we have ensured that 
$\{F\condition F\in A$ and $F$ has a 
nontrivial 
backbone 
    $S\} = 
\{F\condition F\in A$ and $F$ has a 
backbone 
    $S$ with
    $\|S\|\geq\beta\|V(F)\|\}$.
\end{proof}

We now 
address the potential 
concern that the hard instances for the decision problems
we just introduced may be so infrequent
that the relevance of Theorem~\ref{t:backbones-2} 
and Corollary~\ref{c:backbones-1} is undercut.
The following theorem
argues against that possibility by proving that, unless not a 
single $\np$ set
is frequently hard (in the sense made rigorous 
in the theorem's statement), there exist sets of our form that are frequently
hard.  
(This result is making for backbones a point analogous to the one 
our Theorem~\ref{t:freq-backdoors} makes for 
backdoors.  
Hemaspaandra and Narv\'{a}ez~\cite{hem-nar:c:backbones-opacity}
looks at frequency of hardness result for backbones, but 
with results focused on 
$\np \cap \conp$ rather than $\np$.)

\begin{theorem} %
  \label{t:freq-backbones}
  If $h$ is any
  nondecreasing function and for some set $B\in\np$ it holds that each
  polynomial-time 
algorithm
errs with respect to $B$, at
  infinitely many lengths $n$ (resp., for almost every length $n$), on at least
  $h(n)$ of the inputs up to that length, then there will
  exist an $\epsilon>0$ and a set $A\in\p$ of Boolean formulas
  satisfying the conditions of Theorem~\ref{t:backbones-2}, yet being
  such that each polynomial-time algorithm $g$, at infinitely many
  lengths $n$ (resp., for almost every length $n$), will fail to correctly
  determine membership in $L_A$ for at least
  $h(n^\epsilon)$ inputs of length at most $n$.

The same claim also holds for Corollary~\ref{c:backbones-1}.
\end{theorem}

\begin{proof}%
We will prove the theorem's statement regarding 
Theorem~\ref{t:backbones-2}.  It is not hard to also then 
see that the analogous claim holds regarding 
Corollary~\ref{c:backbones-1}.

$B \in \np$ and SAT is NP-complete.  
So let $r_B$ be a
  polynomial-time function, transforming
strings 
into Boolean formulas, such that
(a)~$r_B(x)\in\sat \Leftrightarrow x\in B$, and (b)~$r_B$ is
  one-to-one.   (A construction of such a function is given in 
  Appendix~A of~\cite{hem-nar:t4:backbones-opacity}, and let 
us assume that that construction is used.) As in
the proof of Theorem~\ref{t:backbones-2}, if $F$ is a Boolean formula we define
  $q(F)=\left\lceil\frac{\beta\|V(F)\|}{1-\beta}
  \right\rceil$.  

Without loss of generality, 
we assume that $r_B$ 
outputs only formulas having at least 
one variable.
Note that throughout this proof, $q$ is
applied only to outputs of $r_B$.   
Thus we have ensured that 
none of the logarithms in this proof
have a zero 
as their argument.

Set 
$$A=\{(r_B(x))\wedge(\mathtt{new}_1\wedge\mathtt{new}
    _2\wedge\cdots\wedge\mathtt{new}_{q(r_B(x))}
  )\condition x\in\Sigma^*\}.$$ Because $r_B$ is
computable in polynomial time, there is a polynomial $b$ 
such that for every input $x$ of length at most $n$, the length of
$r_B(x)$ is at most $b(n)$. 
Fix some such polynomial $b$, and let $k$ denote its degree.
In order to find a bound for the length of the added ``tail''
$\mathtt{new}_1\wedge\mathtt{new}_2\wedge
\cdots\wedge\mathtt{new}_{q(r_B(x))}$ in terms of $b(n)$,
notice that the length of the tail is less than some 
constant 
(that holds over all $x$ and $n$, $|x| \leq n$) 
times 
$q(r_B(x))\log q(r_B(x))$. 
Since
$q(r_B(x)) =
\left\lceil\frac{\beta\|V(F)\|}{1-\beta}
  \right\rceil$
and the length of
$r_B(x)$ is at least a constant times the number of its variables, our
assumption that $|r_B(x)|\leq b(n)$ implies the existence of a
constant $c$ 
such that,
for all $x$ and $n$, $|x| \leq n$, we have  
 $q(r_B(x))\leq c\cdot b(n)$. 
Taken together, the two previous sentences 
imply the existence of a constant $d$ 
such that,
for all $x$ and $n$, $|x| \leq n$, we have that 
 the length of
$\mathtt{new}_1\wedge\mathtt{new}_2\wedge\cdots
\wedge\mathtt{new}_{q(r_B(x))}$
is at most $d\cdot b(n) \log(b(n))$, and so certainly is less than 
$d\cdot b^2(n)$.
Let $N$ be a natural number such that, for all $n\geq N$ and 
all $x$, $|x|\leq n$ implies that
$|(r_B(x))\wedge(\mathtt{new}_1\wedge\mathtt{new}
    _2\wedge\cdots\wedge\mathtt{new}_{q(r_B(x))}
  )|\leq n^{2k+1}$; by the previous sentence and the fact 
that $b$ is of degree $k$, such an $N$ will exist.
Let $g$ be a polynomial-time heuristic
for $L_A$.  Notice that $g\circ r_B$---i.e., $g(r_B(\cdot))$---is 
a polynomial-time heuristic
for $B$, since 
$(r_B(x))\wedge(\mathtt{new}_1\wedge\mathtt{new}
  _2\wedge\cdots\wedge\mathtt{new}_{q(r_B(x))})\in
L_A\Leftrightarrow r_B(x)\in\sat$
and  $r_B(x)\in\sat\Leftrightarrow x\in
B$. Let $n_B\geq N$ be such that there is a set of strings
$S_{n_B}\subseteq(\Sigma^*)^{\leq n_B}$,
$\|S_{n_B}\|\geq h(n_B)$, having the property that
for all $x\in S_{n_B}$, $g\circ r_B$ fails to correctly 
determine the membership
of $x$ in $B$. Consequently, there is a set of strings
$T_{n_B}\subseteq(\Sigma^*)^{\leq (n_B)^{2k+1}}$,
$\|T_{n_B}\|\geq h(n_B)$, such that for all
$x\in T_{n_B}$, $g$ fails to correctly 
determine the membership of $x$ in $L_A$;
in particular 
the set
$$T_{n_B}=\{(r_B(x))\wedge(\mathtt{new}_1\wedge\mathtt{new}
    _2\wedge\cdots\wedge\mathtt{new}_{q(r_B(x))})\,
  |\,x\in S_{n_B}\}$$
has this property.

Using the variable renaming $n_A=(n_B)^{2k+1}$, it is now easy 
to see that we have proven that every
length $n_B\geq N$ at which $g\circ r_B$ (viewed as 
a heuristic for $B$) errs on
at least $h(n_B)$ inputs of length up to $n_B$ has a
corresponding length $n_A$ at which $g$ (viewed as a heuristic for $L_A$) errs
on at least $h((n_A)^\frac{1}{2k+1})$ inputs of length up to
$n_A$.  Our hypothesis guarantees the existence of infinitely many
such $n_B\geq N$ (resp., almost all $n\geq N$ can take the role of
$n_B$), each with a corresponding $n_A$.  Setting
$$\epsilon=\frac{1}{2k+1},$$ our theorem is now proven.
\end{proof}

\section{Conclusions}
\label{conclusion}

We constructed easily recognizable families of Boolean formulas that
provide hard instances for decision problems related to backdoors and
backbones under the assumption that $\pneqnp$. 
In particular, we have shown that, under the assumption 
$\pneqnp$, 
there exist easily recognizable families of Boolean formulas with 
easy-to-find strong
  backdoors yet for which it is hard to determine
  whether the formulas are satisfiable.  
Under the same $\pneqnp$ assumption,
we have shown that 
there 
exist 
easily recognizable collections 
of
  Boolean formulas for which it is hard (in fact, NP-complete) to determine whether they have
  a backbone,
and that 
there 
exist 
easily recognizable collections 
of
  Boolean formulas for which it is hard (in fact, NP-complete) to determine whether they have
  a large backbone.
(These results can be taken as 
indicating that, under the very plausible assumption 
that $\pneqnp$, search and decision 
shear apart in complexity for backdoors and
	backbones.  That makes it particularly unfortunate that their
        definitions in the literature are 
        framed in terms of decision rather than search, especially since
        when one tries to put these to work in SAT solvers, it is the 
        search case that one typically tries to use and leverage.)

For both %
our backdoor and backbone results, we have shown that 
if \emph{any} problem $B$ in $\np$ is frequently hard, then there exist
families of Boolean formulas of the sort 
we describe that are hard almost as frequently as $B$.

\subsubsection*{Acknowledgments}  We are grateful to the
anonymous SOFSEM 2019 referees for helpful comments
and suggestions.

\end{document}